\documentclass[6pt, a4paper, twocolumn, showabstract]{naverlabseurope}

\usepackage{lipsum}
\usepackage{tikz,tkz-kiviat,pgfplots}
\usepackage{latexsym}
\usepackage{amssymb}
\usepackage{amsmath}
\usepackage{amsthm}
\usepackage{booktabs}
\usepackage{enumitem}
\usepackage{graphicx}
\usepackage{xcolor}
\usepackage{graphicx}
\usepackage{algorithm} 
\usepackage{algpseudocode}
\usepackage{multirow}
\usepackage{natbib}

\newtheorem{theorem}{Theorem}

\graphicspath{{figs/}}

\title{Multi-Agent Path Finding with Real Robot Dynamics and Interdependent Tasks for Automated Warehouses}

\correspondingauthor{vassilissa.lehoux@naverlabs.com \authsep tomi.silander@naverlabs.com}

\authors{
Vassilissa Lehoux-Lebacque$^{\star}$ \authsep
Tomi Silander$^{\star}$ \authsep
Christelle Loiodice \authsep
Seungjoon Lee \authsep
Albert Wang \authsep
Sofia Michel
}

\affiliations{NAVER LABS}
\contributions{$^{\star}$equal contribution}
\website{}
\websiteref{}

\teaserfig{
}

\begin{abstract}
Multi-Agent Path Finding (MAPF) is an important optimization problem underlying the deployment of robots in automated warehouses and factories. Despite the large body of work on this topic, most approaches make heavy simplifications, both on the environment and the agents, which make the resulting algorithms impractical for real-life scenarios. In this paper, we consider a realistic problem of online order delivery in a warehouse, where a fleet of robots bring the products belonging to each order from shelves to workstations. This creates a stream of inter-dependent pickup and delivery tasks and the associated MAPF problem consists of computing realistic collision-free robot trajectories fulfilling these tasks. To solve this MAPF problem, we propose an extension of the standard Prioritized Planning algorithm to deal with the inter-dependent tasks (Interleaved Prioritized Planning) and a novel Via-Point Star (VP*) algorithm to compute an optimal dynamics-compliant robot trajectory to visit a sequence of goal locations while avoiding moving obstacles. We prove the completeness of our approach and evaluate it in simulation as well as in a real warehouse.

\end{abstract}
\pgfplotsset{compat=1.18}
\begin{document}

\maketitle

\let\thefootnote\relax\footnotetext{This article is an extended version of article 1421 of the ECAI 2024 conference.}
\let\thefootnote\svthefootnote



\section{Introduction}\label{sec:intro}

Multi-Agent Path Finding (MAPF)~\citep{stern_multi-agent_2019} is the problem of planning a set of collision-free paths for a team of agents to reach one, or a sequence of goal locations, with minimal travel time.
With the impressive progress of AI and robotics research over the last decade, an increasing number of real-world applications are based on multi-agent systems and require solving some MAPF problem. Examples include automated warehouses~\citep{wurman_coordinating_2008, li_lifelong_2021-1}, video games~\citep{li_moving_2020}, UAV traffic management~\citep{ho_multi-agent_2019} and 
autonomous vehicles~\citep{dresner_multiagent_2008, li_intersection_2023}. 
While there exists quite extensive literature on MAPF, most works consider a simplified setting where the environment is modeled as a 4-neighbor grid where each agent occupies one cell at a time, and at each discrete time step, can either move to a neighboring cell or wait in place~\citep{stern_multi-agent_2019}. Even in this simplified setting, MAPF is already NP-hard~\citep{yu_planning_2013}.

We target a warehouse scenario in which large robots move heavy objects in a spatially constrained workspace. 
Additionally, we consider that orders are received throughout the day and each order consists of multiple products that robots need to pick up at specific shelves and deliver to a workstation. Each workstation can process only one order at a time.
This creates a stream of inter-dependent pickup and delivery tasks. 
Such a scenario, and its corresponding lifelong MAPF variant, features several characteristics which are not typically taken into account in previous works. 
Due to the weight of the robots, accelerating (or decelerating) to full (or zero) linear and angular speed may take many seconds (and meters). Since transported objects are heavy, the speed and acceleration depend on the load of the robot. 
When path-finding is planned on a graph, a heavy robot cannot necessarily stop from a full speed to the closest node. Furthermore, a large robot also often occupies multiple graph nodes and edges. 
Large robots have seldom space to bypass each other or even turn in place in narrow aisles. 
If the robot can pick up and drop off objects only from one side, the plan should anticipate the turns so that the robots enter the aisles with the right orientation for pickups/drop-offs.
This more realistic setting renders the majority of methods devised for grid-based environments ill-suited for our scenario. Simplified assumptions on robot dynamics also yield trajectories that can quickly lead to collisions when executed by realistic robots (see Section ~\ref{sec:need-for-dynamics}).

In this paper, we introduce a MAPF solution taking the aforementioned characteristics into account. 
Our approach is based on the classical Prioritized Planning (PP) algorithm~\citep{silver_cooperative_2005}, which consists of ordering the agents in a certain priority order, then computing the shortest path for each agent, avoiding the trajectories of the previously planned agents (considered as moving obstacles). 
PP is well-suited to our context as the shortest path computation is done for each robot separately and can in principle accommodate kinematic constraints. 
However, PP is not adapted to handle the collaboration between robots that is needed for our interdependent tasks. We propose an extension of PP, called {\it Interleaved Prioritized Planning}, where the priorities are dynamically assigned throughout the planning process. We prove the completeness of the algorithm under simple assumptions. 
Moreover we introduce a novel shortest path algorithm (named VP$^*$) to compute the optimal trajectory for a robot to visit a sequence of via points, while avoiding moving obstacles and satisfying the kinematic constraints. Similarly to A$^*$ algorithms~\citep{hart_formal_1968}, VP$^*$ is a goal-directed tree search algorithm that relies extensively on a heuristic evaluation of the minimal cost of a subpath to the goal. We obtain this evaluation by computing the optimal robot trajectory when ignoring the collisions. While this simple kinematics-constrained shortest path problem is generally NP-hard~\citep{ardizzoni_solution_2023}, by introducing a tailored routing multi-graph and fixing the robots' speed profiles we obtain a polynomial problem and therefore an efficient evaluation heuristic for VP$^*$.
We evaluate our approach through extensive experiments in simulation and preliminary tests in a real warehouse. 
Those experiments show the necessity of accounting for the real dynamics in our setting and confirm the relevance of the different components of our method.
In particular, we study the impact of different layout configurations (more or less constrained environments) and a more robust version of our method using time margins.

In summary, we propose a novel approach to address the lifelong MAPF problem in a realistic and challenging warehouse setting: with constrained space, interdependent tasks and complex real robot dynamics. Our contributions are the following:
\begin{itemize}
    \item We introduce the Interleaved Prioritized Planning algorithm for MAPF with inter-dependent tasks and prove its completeness under simple assumptions.
    \item We propose the Via-Point Star (VP$^*$) algorithm to compute a shortest single-robot trajectory, satisfying the kinematics constraints, that visits a sequence of locations while avoiding moving obstacles.
    \item We evaluate our approach in simulation and provide ablations of the main components.
    \item Finally we present preliminary results of applying our approach in a real setting.
\end{itemize}

\section{Related Works}
\paragraph{MAPF and MAPD for automated warehouses.}
Our problem is closely related to the lifelong Multi-Agent Pickup and Delivery (MAPD) problem \citep{ma_lifelong_2017}, which is generally treated as a sequence of MAPF problems. There is extensive work on MAPF \citep{silver_cooperative_2005, standley_finding_2010, felner_search-based_2017, stern_multi-agent_2019} with several works focusing on warehouse applications (e.g. \citep{wurman_coordinating_2008, li_lifelong_2021-1}). However all the above approaches use many simplifying assumptions mentioned in Section~\ref{sec:intro}.

\paragraph{MAPF with realistic assumptions.}
In order to deploy MAPF solvers for real-world scenarios, researchers and practitioners have developed two main strategies. The first one is to adapt existing algorithms to handle specific aspects of the real applications. For example, \citet{zhang_efficient_2023} extend some MAPF algorithms to explicitly account for turn actions but in the usual grid-world environment with discretized time steps. \citet{ma_lifelong_2018} present an approach for lifelong MAPD that takes into account the robot's translational and rotational speed but assumes infinite acceleration and deceleration (i.e., no inertia). \citet{li_multi-agent_2019} consider agents of arbitrary shape, but agents still move in unitary time steps from one vertex to another of the graph, hence not accounting for the agents' dynamics. 
The second strategy is to use so-called {\it execution frameworks} which allow to execute a given MAPF solution and update it if necessary. These frameworks are generally agnostic to the underlying MAPF solver. For example, \citet{hoenig_multi-agent_2016} and \citet{honig_persistent_2019} propose post-processing the output of a standard MAPF solver in order to make the trajectories satisfy the kinematics constraints. However the produced plans might not be as effective since these constraints were not taken into account by the planning algorithm.

\paragraph{MAPF in uncertain environments.}
Even with an accurate model of the robots, unexpected time variations in robot actions will happen. To mitigate the effect of this uncertainty, \citet{atzmon_robust_2018} propose $k$-robust MAPF that guarantees the feasibility of the plan even if the agents are delayed by up to $k$ steps.
We propose a similar robustness feature by introducing time margins that ensure robustness of the plan to $k$ seconds of earliness or tardiness of the agents. While historically most of the MAPF research is evaluated on simulators with simplistic assumptions, tests with more advanced simulators such as ROS Gazebo are becoming more common~\citep{varambally_which_2022}. For example, \citet{honig_persistent_2019} conduct a mixed reality tests using iCreate robots. However, unlike in our case, robots' dynamics do not deviate much from the simple assumptions under which the plans are created and the transfer is thus quite successful. In our case, applying these simplifying assumptions quickly lead to collisions (see Section~\ref{sec:need-for-dynamics}).

\section{Lifelong MAPF with Interdependent Tasks}\label{sec:lifelong_MAPF}
In this section, we define more formally our setting and the MAPF problem we want to address.

\paragraph{Agents.}
Agents (or {\it robots}) in this work refer to differential drive robots that can move forward, backward and turn in place. Robots are also equipped with a mechanism to pick up (or drop off) an object at a shelf or workstation, which can only be executed when robot is at a specific orientation and zero speed.
We assume that a dynamics model for the robots is given, so that the trajectory and travel time to go from an initial state (e.g., position and speed) to target state, as well as the times to perform pickups and drop-offs, can be computed.

\paragraph{Warehouse Graph.} We represent the workspace of the warehouse as a general directed graph $G_w = (V_w, A_w)$.
The vertices (or {\it nodes}) are associated with the physical locations where the robot can either stop to pick up and drop off objects (in front of shelves and workstations) or turn (in case there is enough space). The arcs (or {\it edges}) represent the segments to be traveled between the vertices.
Note that we do not make the standard assumption that the graph is a regular grid, therefore not imposing any restriction on the warehouse layout nor the distance between vertices. Such generality leads to cases where a robot may occupy multiple vertices and arcs simultaneously.
Although this graph is not used directly by our routing algorithm, it represents the input data used to build the more sophisticated routing graph described in Section~\ref{sec:routing_graph}.

\paragraph{Orders and Tasks.} We define an {\it order} as a set of objects to be picked up at some shelves and delivered to a (single) workstation; or the reverse, several objects at a workstation that need to be delivered at specific shelves. We assume that orders arrive throughout the day. Formally, an order $o \in \mathcal{O}$ with $k$ objects is defined as a tuple  $(r^o, \{(p_1^o, d_1^o) \dots, (p_k^o, d_k^o)\})$ where $r^o$ is the release date, $p_j^o$ the vertex corresponding to the location of the shelve where the $j$\mbox{-}th object must be picked up or delivered and $d_j^o$ the duration of the corresponding pickup or delivery action.
A duration for the action (pickup or delivery) performed at the workstation for each object can also be defined, either as workstation or object specific.
We define a {\it task} as an agent picking up an object at an initial location and delivering it at a target location. The set of tasks that correspond to a given order are said to be {\it interdependent} because they share the same workstation as initial or target location and must all be finished before a workstation can be used for another order.

\paragraph{Lifelong MAPF with Interdependent Tasks.}
Given a set of agents, a warehouse graph and a stream of orders with their associated tasks, the goal is to find a sequence of collision-free trajectories for the agents to execute their tasks, such that only one order is handled at a workstation at a time and the throughput is maximized.

\begin{figure}[t]
    \begin{center}
        \includegraphics[scale=0.45]{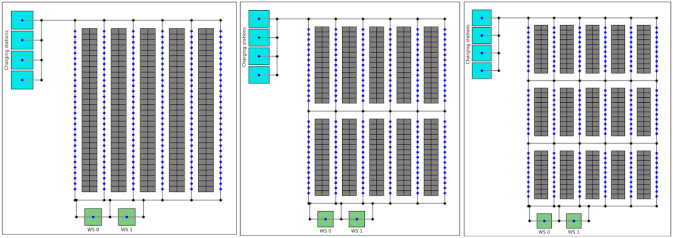}
        \caption{Examples of warehouse layouts with the physical warehouse graph. Shelves are shown in gray, charging stations/waiting areas are shown in cyan and the workstations are shown in green. }\label{fig:layout}
    \end{center}
\end{figure}

\section{Interleaved Prioritized Planning with VP*}
The Lifelong MAPF problem with Interdependent Tasks defined above includes two optimization subproblems: (i) Task Assignment: which agent should execute each task, and which workstation should be used for each task (order)? (ii) MAPF: how should the agents move (in space and time) to execute their tasks? 

For simplicity we consider these two interdependent subproblems separately. In addition, to deal with the stream of online orders, we solve the subproblems at regular, user-specified, intervals. Similarly to Receding Horizon Control \citep{borrelli_predictive_2017}, the idea is to compute a plan and start executing it; then plan again, given the current state of the system, including the newly released orders.
Hence without loss of generality, in the following, we focus on one planning iteration and assume that the orders are known and the objective is to minimize the makespan, i.e., the maximum completion time of all the known tasks.

In this section, we first present our task assignment strategy, and then our MAPF algorithm, Interleaved Prioritized Planning (IPP), that we prove to be complete.
We then describe the details of the Via-Point Star (VP$^*$) algorithm that we use for routing agents without collision, starting by the graph model and collision checking mechanisms that enable us to manage the robots' dynamics.

\subsection{Task Assignment Heuristic}
To optimize the task assignment, we solve the underlying {\it idealized} scheduling problem where we assume that the robots always use the shortest paths (i.e., ignoring potential collisions).
We use a rule-based heuristic, inspired by the classical priority dispatching rule approaches for scheduling \citep{haupt_survey_1989}. 
More precisely, we decompose the system state $S$ into the agents' states $S^\text{agents} := \{(\tau^a, v^a): a \in \mathcal{A}\}$ where $\tau^a$ is the time when agent $a$ will be available and $v^a$ its position (vertex) at $\tau^a$; and the workstations' states $S^\text{ws} := \{\tau^w: w \in \mathcal{W}\}$ with  time $\tau^w$ indicating when the workstation $w$ will be available. 
At the beginning of the day, agents are at their initial position and availability times are zero.
Orders are sorted in a First In First Out (FIFO) fashion, by increasing released date. 
For each order, we first assign the earliest available workstation.
Then, for each task of the order (corresponding to a pickup and delivery), we assign the earliest available agent and update its state, based on the product pickup and drop-off times and the ideal travel times. 
After assigning all the tasks of an order, we update the availability time of the workstation, and switch to the next order. 
To avoid congestion at the workstations, we can limit the number of agents assigned to one order.
From the complete schedule, we extract the assignment of the orders to the workstations and the sequence of tasks assigned to each robot.

\subsection{Interleaved Prioritized Planning Algorithm}\label{sec:IPP}
In the standard Prioritized Planning (PP) algorithm \citep{silver_cooperative_2005}, agents are given a certain priority order, then in descending priority order, we compute the shortest path for each agent, avoiding the trajectories of the previously planned agents (considered as moving obstacles). While PP is well-suited to our context because the shortest path computation is done for each robot separately and can accommodate the kinematic constraints, it is not adapted to handle the interdependence of  tasks. For example, consider a simple scenario where we have 3 robots, 1 workstation and a sequence of 10 orders, each containing 3 products. To process orders as fast as possible, it is natural to divide the tasks (products) of each order between the robots. With PP, after planning the 10 pickups and deliveries of the 1st and 2nd robot, the last planned robot may incur significant delays (due to the numerous moving obstacles). In this case, one cannot guarantee that objects from different orders are not mixed at the workstation. To avoid this issue, we propose an extension of PP, that we call Interleaved Prioritized Planning (IPP), where the priorities are dynamically assigned throughout the planning process, as follows.

Based on the task assignments, we reinitialize the agent and workstation states, and add to each agent's state its assigned sequence of tasks; and for each workstation, its assigned sequence of orders.
For each order (sorted by the FIFO rule), we consider the subset of agents assigned to (the tasks of) this order.
The earliest available agent $a$ gets the priority for path planning. Its path should visit the sequence of goal locations of its next task $t$ and, optionally, finish at its waiting location.
To compute its start time, we take into account the availability time of the relevant workstation.
The start time and list of goal locations are given to the VP$^*$ algorithm (Sec~\ref{sec:shortest_path}) which computes a trajectory that avoids collisions with previously planned trajectories.
Given the trajectory, we update the availability time of the agent and workstation, and remove $t$ from the tasks of agent $a$.
We update the previously planned trajectories with the complete trajectory information, including the last part of the path (that goes to the waiting location).
We repeat until all the tasks of the order are planned and then switch to the next order. 
The last part of each path is optional in the sense that if a robot can directly depart from the delivery location of its previous task to the pickup location of its next task then we discard the go-to-waiting-spot part of its path.
However, this part can be used and is key to ensure that our algorithm always returns a feasible solution.
A detailed pseudo-code for IPP is provided in Algorithm~\ref{algo:mapf}.

\begin{algorithm}
\caption{Interleaved Priority Planning based on VP$^*$ (Section~\ref{sec:shortest_path})}\label{algo:mapf}
\begin{algorithmic}[1]
\State\textbf{Data:} {Neighborhoods $A(\cdot)$}
\State\textbf{Data:} {Precomputed duration of edges $\tau(\cdot, \cdot, \cdot)$}
\State\textbf{Data:} {Path lower bound function $\bar{h}(\cdot, \cdot, \cdot)$}
\State\textbf{Data:} {Penalty function $p(\cdot)$}
\State\textbf{Data:} {Order list $O$ sorted by priority}
\State\textbf{Data:} {Task list $T(o)$ for each order $o \in O$}
\State \Comment{Sorted by planning priority}
\State\textbf{Data:} Assignment of orders to workstations $aw(\cdot)$
\State\textbf{Data:} Assignment of tasks to robots $a(\cdot)$
\State\textbf{Data:} Initial configuration of the robots $l(\cdot)$
\State\textbf{Data:} Waiting place of the robots $wl(\cdot)$
\State \textbf{Data:} $R(\cdot)$ Reservation table
\Comment{Can be empty}
\State $\tau^r(r) \gets 0 \quad \forall r$ \Comment{Next available time for each robot}
\State $paths(r) \gets \emptyset \quad \forall r$
\Comment{Path for each robot}
\For{\textbf{each} $r$ in a given order}
   \State $w\_path \gets$ path to the waiting place with VP*
   \State \Comment{May fail depending on $l$}
   \State Update $R$ with $w\!\_path$
   \State $waiting\_path(r) \gets w\!\_path$
\EndFor
\For{$order \in O$}
\State $ws \gets aw(o)$
\State $T \gets T(o)$
\State Sort $T$ by decreasing priority
\While{$T \neq \emptyset$}
  \State $task \gets T.pop(0)$ 
  \State $r \gets a(task)$
  \If{$task$ or $ws$ is not available at $\tau^r(r)$} 
     \State \Comment{Robot must wait}
     \State $paths(r) \gets paths(r) \cup waiting\_path(r)$ 
     \State Update $\tau^r(r)$
     \State Estimate the necessary waiting time $wt$
     \State Update $R$ and $paths(r)$ with staying at $wl(r)$ for $wt$ starting at $\tau^r(r)$ 

     \State $\tau^r(r) \gets \tau^r(r) + wt$
  \EndIf
  \State $VP \gets ViaPoints(l(r), task, wl(r))$
  \State $path, w\!\_path \gets$ VP$^*$($A, \tau, VP, \bar{h}, p, R, \tau^r(r)$)
  \If{path $\neq \emptyset$}
    \State Update $\tau^r(r)$ 
    \Comment{{\small Time at task's end location}}
    \State $paths(r) \gets paths(r) \cup path$
    \State $waiting\_path(r) \gets w\!\_path$
    \State Reserve $path \cup w\!\_path$ in~$R$
  \Else
    \State \textbf{return} $\emptyset$
  \EndIf
  \State Update $T$'s order
\EndWhile
\State Update $ws$'s available time 
\State \Comment{End time of the last action of $o$ on $ws$}
\EndFor
\State \textbf{return} paths
\end{algorithmic}
\end{algorithm}

We now state the completeness of the proposed algorithm.

\begin{theorem}\label{th:complete}
Assume that each robot has a designated waiting place,
where it can be idle without interfering with other robot trajectories. If the robots are at their waiting place at the beginning of the planning, the Interleaved Prioritized Planning algorithm is complete.
\end{theorem}
\begin{proof}
To show that the algorithm is complete, we need to prove that it always returns a feasible solution.
There are two constraints to consider: (i) the paths of two robots must not collide and (ii) all the tasks of an order must be finished on its workstation before any new task can start on that workstation.
The robots are starting at their waiting place. 
The planned trajectories contain, for each robot, a reservation of its waiting place for the duration of the planning horizon and no other information.
Each robot can hence wait for an unlimited time at its starting position without blocking the way of the other robots.
Consider the first task $o_{1, 1}$ of the first order $o_1$ assigned to robot~$r$, that is released at time $\tau^o_1$.
Its assigned workstation~$w$ is either available for order $o_1$ or will be available at a given time $\tau^w$.
We can compute, given $r$'s availability time~$\tau^r$, the earliest instant $\tau$ at which it can start to perform the task to respect the availability constraints, assuming the agent does not wait once it has left its waiting place, using a classical shortest path algorithm or via-point search.
Algorithm~\ref{algo:mapf} plans the robot to wait until~$\tau$ and then computes a path starting at~$\tau$ for the via-points for the task plus a return for the robot to its waiting place.
As there is no other robot for which we planned a move, a path that respects constraint~(i) always exists, assuming the graph is connected. By definition of~$\tau$, it also respects constraint~(ii).
We update the available time~$\tau^r$ of the robot with the task's end time, so that it can restart from its last position.
If the order has other tasks, we use the same availability time~$\tau^w$ for the workstation when determining when to start from the waiting place. The time~$\tau$ computed is now an optimistic estimate based on travel times without collision from the robot's current location. It ensures~(ii) as the robot can only arrive later than $\tau^w$ at the workstation if it needs to avoid other robots.
If the robot needs to wait, we send it to wait at its waiting place: if it is different from its current position, we reserved a path to go there during the planning of the robot's previous task.
Constraint~(i) is then ensured by the existence of a feasible path: the robot can always wait for all the planned trajectories to be finished
to move. 
When all tasks of an order have been planned, the availability time of the workstation is updated with the end time of the last action of the order, ensuring that no task of another order is able to start before that time.
Tasks of the next order can hence be planned following the same logic and respecting the constraints.
\end{proof}
Note that Theorem~\ref{th:complete} also holds when the robots start at locations other than the waiting places and either a path is reserved in the planned trajectories from their current position to their waiting places, or such a path can be computed for each agent in a priority order given to the algorithm for initializing the planned trajectories.

\subsection{Routing Multi-Graph}\label{sec:routing_graph}
In many other works, routing is performed directly in the warehouse graph~$G_w$ as defined in Section~\ref{sec:lifelong_MAPF}.
However, as a robot may occupy several nodes in our context, cannot turn everywhere and has more complex dynamics, we design a specific routing graph that can integrate those elements in its structure. 
To do so, we convert the warehouse graph $G_w$, to a directed routing multi-graph in which each vertex (say node C in the Figure~\ref{fig:nodes}) of $G_w$ is represented with several vertices of different types and their associated arcs: 
(i) a start and a stop node, to start planning from that position or end an itinerary at that position (only shown for node A below);
(ii) for each adjacent node in the warehouse graph (like node D), we create two nodes: one for arriving from that node and one for going to that node. 
They are linked by an additional arc allowing for going backward without changing direction;
(iii) where it is possible to turn, we add turning arcs for the robot to reverse its orientation (i.e., a U-turn);
(iv) at the nodes (like C) where two nonparallel edges meet, we add two turning arcs per node-triplet (two for B-C-D, and two for D-C-B), 
one for turning clockwise and one for turning counter-clockwise. One of these arcs also turns the traveling orientation of the robot. 

\begin{figure}[ht]
\begin{center}
\includegraphics[keepaspectratio, width=0.9\columnwidth]{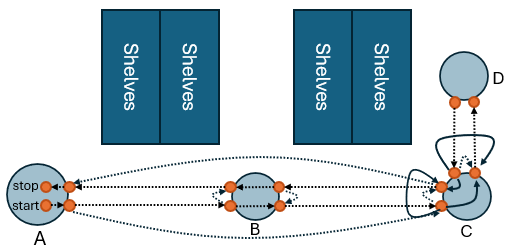}
\end{center}
\caption{Example of graph conversion. Original warehouse nodes are in blue, new routing graph nodes are in orange, dashed lines are straight arcs, while solid lines are turning arcs. The solid turning lines within node C represent two arcs each: for turning clockwise and counter-clockwise.}
\label{fig:nodes} 
\end{figure}

The robot needs to stop in order to turn, to avoid sliding, tilting and dropping its potentially heavy load.
However, if the robot is going straight, we want to traverse long paths without stopping, possibly passing through several nodes. 
To this effect we augment the routing graph with additional shortcut edges between nodes that are on a straight line.
With shortcuts in place, we can set the initial and final linear and angular velocities at all the nodes to zero. We use the robot's dynamics to associate two travel times with each arc: one for an empty robot and one for a loaded one. The times for each arc traversal can now be precomputed based on accelerating the robot to maximum linear/angular velocity within the edge-specific speed limits and decelerating back to speed zero.
As the speed at each node is zero, and arc travel times are fixed, computing the shortest path without moving obstacles for a single robot is polynomial (as opposed to the general case of shortest path with kinematics~\cite{ardizzoni_solution_2023}). 
We use this property to rapidly compute the lower bounds on the duration of paths that avoid collisions.

\subsection{Collision Checking} 
\label{sec:collision-checking}
In classical settings, strong hypotheses on the routing graph are taken to avoid robot collision.
Robots in adjacent nodes cannot collide, and many rules on robot movement ensure that collisions are avoided~\cite{stern_multi-agent_2019}.
In our setting, adjacent nodes can be so close that two robots occupying them would collide, or so far away that two robots traveling the same edge would not collide.
We hence define a more spatially explicit collision checking method, adapted to this setting and reminiscent to swept AABB~\cite{gottschalk_collision_1998}.

The position $xy \in \mathbb{R}^2$ of each robot and its orientation (yaw) $\theta \in [-\pi, \pi]$ is uniquely determined at each time~$\tau$.
In the routing algorithm, the  \textit{configuration} of the robot is defined as a tuple $(\tau, xy, \theta, s, s_{\theta}, l)$, 
where $s\in \mathbb{R}$ is the linear velocity in direction $\theta$, $s_{\theta} \in \mathbb{R}$ is the angular velocity, and $l=1$ if agent is loaded and $l=0$ otherwise.
For collision detection, the space occupied by the robot is modeled as a connected two-dimensional set (such as a polygon)  centered at its $xy$-position and rotated by its orientation. 
Such a set may include additional padding. We check if the spaces occupied by two different robots intersect at any given time~$\tau$. 
For a practical implementation, one may, for example, cut continuous time into intervals, and model the space occupied by the robot during an interval as (a convenient superset of) the union of the sets the robot occupies during this interval. Checking if a robot would collide with other robots at time $\tau$ can then be conservatively approximated by checking if the space occupied by the robot intersects with the spaces of the other robots occupied during the interval containing $\tau$. One may also add some time safety margin by reserving some neighboring intervals.

\subsection{VP$^*$ Algorithm}\label{sec:shortest_path}

To plan the displacements of the robot for a given task, classical shortest path algorithms are inappropriate as we need to compute paths that take into account not only the position of static obstacles (modeled into the graph with non-existing edges) but the position of the other robots, which are moving in the warehouse at the same time as we plan the robot's path.

\paragraph{Via points.}
For each task, we must plan several consecutive displacements of the robot (e.g., going to an aisle to pick up an object, going to the workstation to deliver the object, moving to a waiting area to wait for the next task).
We call \textit{via points} the locations that we have to visit sequentially.
If we were to use classical shortest path with collision avoidance sequentially for each pair of consecutive via points, we would have no guarantee of finding a path for all the displacements of the sequence (see Section~\ref{sec:ablation}).
Indeed, taking the shortest path for the first pair of the sequence may prevent the robot finding any feasible path to the remaining via points.

\citet{li_lifelong_2021-1} propose a complete heuristic algorithm for a simpler version of the shortest path with via points problem.
This heuristic however does not account for robot dynamics, actual robot load, position and orientation in space, and its collision checking is based on vertex and swapping conflicts only.
We explain below how to handle these additional constraints and propose an improvement that guides the search faster toward a first feasible solution.

\paragraph{Robot dynamics in VP$^*$.}
Even if the arc weights are precomputed, based on the robot's dynamics, we need precise displacement information when checking collision for the current robot against already planned robot paths (our moving obstacles), which impacts the speed of the algorithm.
From a practical point of view, the space occupied by the robots along their planned paths are computed and stored in a \textit{reservation table} for later collision checking when other robots compute their paths.

\paragraph{Robot configuration.}
In our context, the orientation of the robot is important: in order to pick up or drop off objects at the shelves and workstations, the robot must be in an specific orientation to execute the (pickup or drop-off) action.
However, due to space constraints, it may not be possible for the robot to rotate in the aisle, so it must be in the right orientation when entering the aisle.
Orientation is hence taken into account in the configuration of the robot, and also in the description of the via points and in the routing multi-graph.
In addition, the speed of the robot is different when loaded, so the load is also added to the via-point information.
Each via point is defined by a node in the graph (and physical position), a partial \textit{in-configuration} indicating the required orientation (yaw) and the load when arriving at the node, and a partial \textit{out-configuration} indicating the required orientation and the load when leaving the node, and the time required at the via point to perform the associated action.
The via points are the only points where the robot can perform an action.

\paragraph{Shortest path with via points.}
The objective is to find an earliest arrival path starting from the first via point and then passing by each via point in the order of the sequence, and arriving at the last via point while respecting the constraints imposed by the partial configurations of each via point, stopping at each via point for the required time, and avoiding collision with the moving obstacles whose configuration is known at each time instant from the reservation table.
Although, in our warehouse, the only moving obstacles are the other robots, the collision checking could be performed similarly for any kind of moving obstacle whose occupied spaces have been stored in the reservation table.
In our context, travel time must be either continuous or discrete with high precision to avoid cumulative errors over multiple paths computed sequentially during the planning.
However, to reduce the computation time, when the robot is waiting at a given point, the waiting time is discretized to the second in our experiments.
When waiting at a node is possible, the graph contains an additional arc from the node to itself with a duration equal to the chosen minimum waiting time.

\paragraph{Notation} 

We denote by $G = (V, A)$ the routing graph (as explained in Section~\ref{sec:routing_graph}) where $V$ is the set of nodes and $A$ the set of arcs between those nodes.
For a given node~$v$, $A(v)$ denotes the output arcs of~$v$, i.e., the set of arcs $\{(v, v') \in A | v' \in V\}$.
The weight $\tau(v, v', l)$ of an arc $(v, v')$ is the duration to traverse the arc~$(v, v')$ with $l=1$ if the robot is loaded and 0 otherwise.
The angle $\theta(v, v')$ is the rotation angle of the edge, when the edge allows to turn in a given direction, and 0 otherwise.
A path is a succession $(v_0, v_1), (v_1, v_2), \dots (v_{i-1}, v_i)$ of arcs of~$A$.
The duration of a path~$p$ is $\sum_{(v, v') \in p} \tau(v, v', l)$.
To obtain a lower bound  $h(v, v', l)$ for the duration of a path between node $v$ and node $v'$, we precompute the duration of the shortest path under the assumption that there are no moving obstacles (one duration for a loaded robot and one for an unloaded robot), and without any requirement on orientation at node~$v'$, using any classical shortest path algorithm, like the one used by~\citet{li_lifelong_2021-1}.
As stated before, this shortest path can be obtained in polynomial time based on our routing graph.
We denote via points with nodes $v$ with additional attributes so that $v.l$ is the out-configuration load of via point $v$ and $v.d$ is the duration to perform an action at via point $v$.
When passing by~$k+1$ sequential via points $\bar{v} = (v_0, v_1, \dots, v_k)$, we define the estimate of the duration of the path to reach the last via point as
as the sum of the estimates of the subpaths, i.e., $\bar{h}(\bar{v}) = \sum_{ 0 \leq i < k} \left(h(v_i,v_{i+1},v_i.l) + v_{i}.d\right)$ (it does not include the time of the action at the last via point).
With some modifications to check the required orientation at nodes, configurations at start point and via points can be taken into account to obtain a better estimate.
For a partial solution where via points $(v_0, v_1, \dots, v_{vp})$, $vp < k$, have already been passed sequentially, we can then estimate the path duration from its last node~$v$ in configuration~$c$ to via point $v_k$ similarly.
This duration is denoted~$\bar{h}(v, c, vp, \bar{v})$ or $\bar{h}(v, c, vp)$ for short.

\paragraph{VP$^*$ algorithm.}
The VP{*} algorithm plans the path from node $v_0$ to $v_k$ sequentially visiting via points $\bar{v}$ and avoiding collisions with previously planned paths saved in the reservation table. 
The algorithm iterates on a priority queue, which directs the search toward the most promising elements.
This queue contains tuples~$(v, c, vp, hs)$, where $v$ is a node, $c$ a configuration at this node, $vp$ the index of the last via point reached and $hs$ the {\it heap score} of the tuple. 

At the beginning of the search, we initialize this queue with a tuple $(v_0, c_0, vp_0, hs_0)$ containing the origin node $v_0$ (the first via point), the initial configuration~$c_0$ (obtained from the partial out-configuration of the first via point and the path start time), $vp_0 = 0$ and the heap score~$hs_0$ of the element.

During a search step, the algorithm pops a heap element $(v, c, vp, hs)$ with the lowest heap score and explores its neighborhood~$A(v)$ in the routing graph. 
We first check the lower bound $\bar{h}(v, c, vp)$ on the duration of a path to the destination $v_k$ from $(v, c, vp)$.
If it is larger than the duration of the best path found so far, the search is pruned by going to the next step.
If not, we check if we have reached the next via point in the right orientation, i.e. if node~$v$ is the next via point and if configuration~$c$ has the same orientation as the in-configuration of the next via point.

If we are at the next via point, we need to be able to stay at node~$v$ without collision for the duration of the action to perform at this via point.
If we can, and we are at destination $v_k$, we can update the earliest known arrival time at node~$v_k$ and go to the next step. 
If we are not yet at final destination $v_k$, we increase by one the index~$vp$ of the last reached via point and update the configuration time~$\tau$ with the time spent at the via point before exploring the neighborhood of~$v$, as well as the load from the out-configuration of the via point.

For each neighbor~$v'$ of node~$v$, we verify if we can use arc $(v, v')$ starting at time~$\tau$ without collision, using the dynamics of the robot to compute its position in time and space while traveling the edge.
If there is no collision, we check that the obtained configuration $c'$ at~$v'$ has not already been added to the queue for via-point index~$vp$.
If not, we compute the heap score $hs'$ and add the new element $(v', c', vp, hs')$ to the queue.

The algorithm finishes when the queue is empty.
Note that additional stopping criteria could be added, such as having found a first feasible path, having popped a maximum number of elements out of the queue, or having spent a certain amount of time.
The first drops the optimality of the algorithm, the last two leads to drop the completeness of the algorithm as it can return before finding a first feasible solution when one exists.
Algorithm~\ref{algo:exact} gives the pseudo-code of the algorithm.

\begin{algorithm} 
\caption{Shortest path with via points and time dependent obstacle collision avoidance VP$^*$}
\label{algo:exact}
\begin{algorithmic}[1]
\State\textbf{Data:} {Neighbourhoods $A(\cdot)$}
\State\textbf{Data:} {Precomputed duration of the edges $\tau(\cdot, \cdot, \cdot)$}
\State\textbf{Data:} {Via points $\!VP\!=\!(v_i,\! c^{in}_i, c^{out}_i,\! d_i)_{0 \leq i < k}$}
\State\textbf{Data:} {Path lower bound function $\bar{h}(\cdot, \cdot, \cdot)$}
\State\textbf{Data:} {Penalty function $p(\cdot)$}
\State\textbf{Data:} Reservation table $R$
\State\textbf{Data:} {Start time $\tau_0$}
\State $Prev\_c \gets Dict()$  \Comment{Previous configs in the path}
\State $min\_arr \gets \infty$   \Comment{Min path duration found so far}
\State $min\_c \gets \emptyset $ \Comment{Arrival config of best path so far}
\State $c_0 \gets (\tau_0, xy(v_0), c^{out}_0.\theta, 0, 0, c^{out}_0.l)$
\State $hs \gets \bar{h}(v_0, c_0, 0) + p(0)$
\State $Q \gets \{ (v_0, c_0, 0, hs) \}$

\While{$Q \neq \emptyset$}
  \State $(v, c, vp, hs) \gets Q.pop()$  \Comment{Min score element}
  \If{$c.\tau + \bar{h}(v, c, vp) > min\_arr$}
    \State \Comment{Cannot improve on current best solution}
    \State continue 
  \EndIf  
  
  \If{at\_destination}
    \If {robot can stay at $v_{k-1}$ for $d_{k-1}$} \State \Comment{Checking $R$}
    	  \State Update $min\_arr$, $min\_c$
    	  \State \textbf{continue}
    	 \EndIf
  \Else 
  \If{ at\_next\_via\_point}
    \If {robot can stay at $v_{vp}$ for $d_{vp}$} \State\Comment{Checking $R$}
      \State $c.\tau \gets c.\tau + d_{vp}$ \Comment{Staying at $v_{vp}$}
      \State Update $c$ with $c^{out}_{vp}$ 
      \State $vp \gets vp + 1$
    \EndIf 
  \EndIf
  \EndIf

  \For{$(v, w) \in A(v)$}
    \If{taking $(v, v')$ at $c.\tau$ in configuration $c$ does not collide with any planned path of $R$}
      \State {\small $c' \gets$ Update $c$ with $\tau(v, v', l_{vp})$ and  $\theta(v, v')$ }
      \If{$(v', c', vp, hs)$ not added to $Q$ before}
        \State $hs \gets c'.\tau + \bar{h}(v', c', vp) + p(vp)$
        \State Add $\{ (v', c, vp, hs) \}$ to $Q$
        \State Update $Prev\_c$
      \EndIf
    \EndIf
  \EndFor  
\EndWhile

\State \textbf{return} ReconstructPath($Prev\_c$, $min\_c$)

\end{algorithmic}
\end{algorithm}

Note that in our experiments, we stop the search as soon as a first solution is found.
In practice, it is much faster (as the algorithm does not need to empty the queue completely before returning a solution) and we observed on our setting good results at the level of the MAPF algorithm, sometimes even better than with optimal paths, suggesting that optimizing locally the path of a given robot might in some cases make it more difficult for other agents to plan their own paths.

\paragraph{Heap score.}
The heap score of $(v, c, vp)$, where~$v$ is a node, $c$ a configuration and $vp$ a via point index, is computed as the sum of an optimistic estimate~$\bar{h}(v, c, vp)$ of the shortest path duration to the destination after passing through all via points and a penalty~$p(vp)$ that is higher when the number of remaining via points to pass is higher, and equal to 0 when only one via point remains to be reached.
This penalty aims to favor finding a first feasible path between origin and destination passing through all via points by making a depth-first like search.
It is hence different from the heap score used by ~\citet{li_lifelong_2021-1} and of the classical A$^*$ algorithm as the score is not necessarily a lower bound on the duration of a path to destination from the current heap element. 
However, as we use the travel-time lower bounds to prune the search and not the heap scores, the algorithm remains optimal if run until the queue is empty.
Section~\ref{sec:ablation} compares our heap score to the one with $p(vp) = 0$, and shows that the proposed penalty is indeed an efficient way to reduce the search space.

\paragraph{Further improvements of VP*}

As the robots must regularly wait to avoid collision in our setting, the VP* algorithm, implemented as described above, might lead to solutions that are optimal in terms of path duration, but not satisfactory for a human observer.
Indeed, spending the waiting time immobile or moving is equivalent according to the path duration criterion, and we observe in our experiments solutions with unnecessary moves happening during those waiting times.
Because we discretize waiting times (by setting the duration of the loop arcs of the graph), it can also happen that those moves improve the minimum path duration by allowing to return at a node slightly before what would have been possible staying immobile.
To avoid this problem, we introduce a secondary criterion, the minimum total time in movement of the robot, that we use to break ties between equivalent solutions, and we discretize arrival times in such way that solutions that are very close in terms of arrival times are deemed equal according to this criterion.
Then, among equivalent solutions for this first criterion, only the solution with the lowest total moving time is kept.

More details on how to implement this version of the algorithm and associated experiments can be found in Appendix.
Due to the discretization and to the ties breaking properties of the secondary criterion, we obtain solutions of equivalent quality with lower execution times. 

\section{Validation}\label{sec:experiments}

\subsection{In Simulation}
We implement a simulator of the warehouse environment where robot movements are modeled with simple dynamics. 
The robot's linear acceleration $A$ is constrained to the interval $[A_{dec}, A_{acc}]$. 
The value of $A$ may depend on the loading state of the robot since a heavy robot cannot accelerate as fast as a lighter one. 
If we suppose that the robot is always accelerating or decelerating as fast as it can toward the desired speed, we can compute the shortest time $t_s$ it takes to travel a line segment of length $d$ when starting with velocity $v_i$ and obeying acceleration constraints, segment-specific speed limit $V$ and the maximum speed at the end of the segment~$V_f$. Identical computations can be made for pure turning arcs in which only the robot's orientation~$\theta$ changes.
In this case the same computation are performed with acceleration limits $[A_{\theta_{dec}}, A_{\theta_{acc}}]$, angular distance $d_{\theta}$, and maximum angular speed $V_{\theta}$.
For all the arcs in the routing graph, the initial and maximum final velocities equal zero.
Under those hypotheses, we get highly accurate travel times, compared to the actual measurements in the warehouse (see Section~\ref{sec:real}).

In the simulation experiments, the maximum linear/angular velocities are 0.2[m/s; rad/s], and the constant (ac/de)celerations is 0.25[m/s$^2$; rad/s$^2$] when the robot is loaded, 0.5 otherwise.

We consider three generic warehouse layouts, with one to three rows of shelves, as illustrated in Fig.~\ref{fig:layout}. 
For each layout graph, we generate 100 instances with 5 orders each.
For each order, the number of tasks is sampled from a max-inflated truncated geometric distribution with mean=2.5 and max=4.
The order being a delivery or a pickup is chosen uniformly between the two options and the location of each task is sampled uniformly from all the possible positions on the graph (in our layouts, at the nodes of the shelves).

Experiments use an Intel(R) Xeon(R) CPU E5-2680 v4 2.40GHz x86\_64 server with 56 cores, 2 threads per core, and 1.0TB of RAM.

\subsubsection{The Need for Dynamics}
\label{sec:need-for-dynamics}
To study the need for taking dynamics into account, we compare our proposed method to two different baselines. The first baseline, admittedly naive given our graph with edges of different length, assumes that we can move to any neighboring node in constant time. Following the paths planned based under this hypothesis causes collisions in our standard warehouses after following two edges. 

The second baseline, like suggested by~\citet{ma_lifelong_2018} ignores inertial dynamics and allows instantaneous switching between stopping and constant-speed movement. This assumption, does not actually simplify or speed-up our algorithm (since the spatially explicit collision checking is still needed). We simulated 100 different scenarios and the paths planned with this assumption using 4 robots. All of them resulted in a collision within 5 minutes, half of them failed within 1 minute, some within the first 10 seconds.
Due to the increased need for replanning and no significant speed-up benefit, it is clear that dynamics should be considered in the planning process.

\subsubsection{Ablation Study}\label{sec:ablation}

In this study, we use a simple layout graph similar to the left layout in Fig.~\ref{fig:layout}, with 100 nodes and 106 arcs in the warehouse graph and 596 nodes and 1908 arcs in the routing multi-graph. We call it \textit{Layout 0}.

\paragraph{Sequential planning vs via points.}
To evaluate the need for the proposed via-point algorithm VP$^*$ in our IPP MAPF solver, we replace it by a more naive approach that sequentially plans the shortest paths between the different via-points pairs of a task, while avoiding collisions.
For this sequential approach, when computing a path between two via points, the heap score contains only an optimistic estimate of the shortest path duration to the second via point, as in the A* algorithm.
We tested the resulting IPP with sequential planning on Layout 0, using 100 scenarios and 2 to 4 robots. It fails to find a feasible solution for 3 instances out of 100 with 2 robots, for 6/100 instances with 3 robots and for 11/100 instances with 4 robots.
The increase in the number of failures with the number of robots is expected as more robots in the tight environment means more constraints on collision avoidance. This confirms the necessity of our via-point planning strategy.

\paragraph{Without reserving a path to the waiting place.}

Another important factor ensuring completeness is the reservation of paths to the waiting places of the robots in the reservation table.
If those paths are no longer reserved, but only computed when a robot needs to go and wait there, with Layout 0, Algorithm~\ref{algo:mapf} fails to find a feasible solution in 18 instances out of 100 for 2 robots, for 48 instances for 3 robots and for 62 instances for 4 robots.
As for sequential planning, the increase of the number of failures with the number of robots is expected.
Although the robots do not always need to use the waiting paths, reserving them ensures that the robots can actually leave their last location, and possibly use a part of the reserved path to reach their next destination.
Without planning and reserving paths to the next via point(s), a robot can easily find itself blocked or overrun by other robots who can plan their routes taking advantage of the first robot not making any reservations for its future.

\paragraph{Impact of the penalty function in VP$^*$.}

As explained in Section~\ref{sec:shortest_path}, we change the standard heap score of states in our VP$^*$ algorithm by adding to the path's duration estimate a penalty based on the number of via points that have not been visited yet.
In our implementation, this penalty is simply the number of remaining via points minus 1 times 1000 seconds.

Table~\ref{tab:penalty} shows that this addition is fulfilling its objective: the number of states visited is significantly reduced and the algorithm run time decreases accordingly, while the impact on the plan quality is very low.
The run time of the plans is improved by a median factor of 3.2 (minimum is 0.73, maximum is 25.39 on our set of instances).

\begin{table}
    \setlength{\tabcolsep}{4pt}
    \centering
\resizebox{0.485\textwidth}{!}{%
  \begin{tabular}{r|r||r|r||r|r||r|r}
    \multicolumn{2}{c||}{{}}& \multicolumn{2}{|c||}{{2 robots}} & \multicolumn{2}{|c||}{{3 robots}} & \multicolumn{2}{|c}{{4 robots}}\\
    \hline
    \multicolumn{2}{c||}{{Heap score}}& std & penalty & std & penalty & std & penalty\\
    \hline
    \multirow{3}{0.9cm}{Visited states ($\times 10^3$)} 
& Q1 & 13 & 5 & 58 & 21 & 137 & 46\\
& Q2 & 30 & 10 & 109 & 75 & 303 & 74 \\
& Q3 & 76 & 22 & 184 & 122 & 485 & 130 \\ \hline
    \multirow{3}{0.9cm}{{Run time (s)}} & Q1 & 24 & 9 & 110 & 38 & 276 & 90  \\ 
& Q2 &  55 & 17 & 212 & 75 & 611 & 137 \\ 
& Q3 &  131 & 39 & 355 & 122 & 1009 & 289 \\ \hline
    \multirow{3}{0.9cm}{Make-span (min)} & Q1 & 18 & 18 & 15 & 15 & 13 & 13   \\ 
& Q2 & 22 & 22 & 17 & 17 & 15 & 15 \\ 
& Q3 &  25 & 25 & 19 & 19 & 15 & 15 \\ 
  \end{tabular}}
    \caption{Quartiles of number of states visited, run-time and makespan when using standard vs penalty augmented heap score}
    \label{tab:penalty}
\end{table}

\subsection{Impact of the Layout}

In our constrained environment where robots cannot turn to reverse their orientation or pass another robot, the layout may impact the efficiency of the planning.
To measure how it influences the different metrics, we run experiments on 3 different layouts of approximately the same size (220 to 232 nodes and 226 to 248 arcs in the warehouse graph), with the same number of pickup and delivery locations but different number of corridors that are represented on Fig.~\ref{fig:layout}.
Layout 1, on the left, is the more constrained environment, where long aisles do not offer any possibility for two robots to pass while Layout 3 (on the right side) has two intermediate corridors that allows for the robots to avoid one another more easily when moving in opposite directions.
The total number of shelf nodes in the warehouse graph is 180 for the 3 graphs.
For this experiment, we use 50 scenarios where the positions of the pickups and drop-offs at the shelves are identical (if corridors splitting the shelves are not taken into account) on the 3 layouts.
Table~\ref{table:layout} shows that the layout does not impact much the makespan.
To study the impact on individual orders, rather than maximum completion time, we add a regret metric measuring the difference between the actual processing time of the tasks and the estimated processing time computed without collision. 
It measures how much the robots need to diverge from an ideal path in order to avoid one another.
The layouts that allow for easier collision avoidance have reduced regret for all number of robots, Layout 3 improving on Layout 2 and Layout 2 improving on Layout 1.
As for the run time, it is much more impacted by the number of robots than by the layout, which is expected as the graphs have very similar sizes.

\begin{table}
  \setlength{\tabcolsep}{3pt}
  \centering
\resizebox{0.485\textwidth}{!}{%
  \begin{tabular}{r|r||r|r|r||r|r|r||r|r|r}
 
    \multicolumn{2}{c||}{}& \multicolumn{3}{|c||}{{2 robots}} & \multicolumn{3}{|c||}{{3 robots}} & \multicolumn{3}{|c}{{4 robots}}\\
    \hline
    \multicolumn{2}{c||}{Layout}& {1} & {2} & {3} & {1} & {2} & {3} & {1} & {2} & {3}\\
    \hline
    \multirow{3}{0.9cm}{Make-span (min)} 
& Q1 & 23 & 22 & 23 & 19 & 17 & 19 & 16 & 16 & 16\\
& Q2 & 28 & 29 & 30 & 21 & 21 & 22 & 19 & 19 & 19\\
& Q3 & 32 & 32 & 33 & 25 & 24 & 25 & 23 & 23 & 24\\ \hline
    \multirow{3}{0.9cm}{{Regret (\%)}} & Q1 & 4.2 & 3.1 & 2.9 & 7.6 & 6.4 & 6.1 & 10.7 & 8.2 & 7.4\\
    & Q2 & 5.7 & 5.0 & 4.0 & 9.9 & 7.7 & 7.6 & 12.5 & 10.1 &  9.3\\
    & Q3 & 8.1 & 6.3 & 5.5 & 13.8 & 9.9 & 9.1 & 15.4 & 12.1 & 10.8\\ \hline
    \multirow{3}{0.9cm}{{Run time (s)}} & Q1 & 51 & 54 & 25 & 154 & 129 & 165 & 596 & 534 & 524\\ 
& Q2 & 101 & 95 & 76 & 326 & 343 & 403 & 1163& 1051 & 1264 \\ 
& Q3 & 249 & 250 & 257 & 815 & 742 & 799 & 1679 & 1876 & 2226 \\ 
  \end{tabular}}
  \caption{Quartiles of the makespan (in minutes), regret (in percent) and run time (in seconds) of the IPP algorithm for Layouts 1, 2 and~3.}\label{table:layout}
\end{table}

\vspace{-0.5cm}
\subsection{Dealing with Uncertainty}
With a large robot, actuation noise is likely to make the movement differ from the ideal trajectory.
This may lead to collisions especially if the actual movement is often slower or faster than the ideal movement. To reduce the problem created by such noise, we study the efficiency of reserving neighboring time intervals around the actual time for each position of the robots in the reservation table.
When we reserve an area for the robot for a given instant~$\tau$, we also reserve it for~$\Delta$ seconds before and after~$\tau$.
We test the robustness of plans computed in Layout 0 with different values of~$\Delta$ in a stochastic environment where edge travel times are scaled by a factor that is PERT-distributed with the support of [1.0, 1.1] and with mode at 1.01.
For all numbers of robots, adding more time margin in planning increases the time to collision (see Table~\ref{table:margins}).
Unlike makespan that is largely unaffected by the different levels of margin, the planning time clearly increases with bigger margins. The supplementary material contains a video featuring 3s time margins for four robots.

\begin{table}
  \setlength{\tabcolsep}{4pt}
  \centering
\resizebox{0.485\textwidth}{!}{%
  \begin{tabular}{r|r||r|r|r||r|r|r||r|r|r}    
    \multicolumn{2}{c||}{{}}& \multicolumn{3}{|c||}{{2} robots} & \multicolumn{3}{|c||}{{3} robots} & \multicolumn{3}{|c}{{4} robots}\\
    \hline
    \multicolumn{2}{c||}{{$\Delta$ (s)}}& {0} & {2} & {4} & {0} & {2} & {4} & {0} & {2} & {4}\\
    \hline
    \multirow{3}{1.2cm}{{Time to failure $\;$ (min)}} 
& Q1 & 3 & 5 & 6 & 2 & 3 & 5 & 2 & 3 & 4 \\
& Q2 & 6 & 7 & 9 & 3 & 5 & 7 & 2 & 4 & 6 \\
& Q3 & 10 & 10 & 14 & 4 & 8 & 9 & 3 & 5 & 8 \\ \hline
    
    \multirow{3}{1.2cm}{{Make-span (min)}} & Q1 & 18 & 19 & 19 & 15 & 15 & 15 & 13 & 13 & 13 \\ 
& Q2 & 22 & 22 & 22 & 17 & 17 & 18 & 15 & 15 & 15 \\ 
& Q3 & 25 & 25 & 26 & 19 & 20 & 20 & 17 & 18 & 17 \\ \hline
    \multirow{3}{1.2cm}{{Run time (s)}} & Q1 &  9 & 13 & 18 & 38 & 67 & 106 & 90 & 173 & 314  \\ 
& Q2 & 17 & 26 & 38 & 75 & 140 & 228 & 137 & 366 & 592 \\ 
& Q3 & 39 & 56 & 85 & 122 & 226 & 367 & 289 & 791 & 1330  \\ 
  \end{tabular}}
  \caption{Quartiles of time to collision when applying the plans in a stochastic environment, makespan of the plans and run time of the IPP.}
  \label{table:margins}
\end{table}

\subsection{Scaling to larger instances}

In this section, we evaluate the performance of the algorithm on larger instances with 10 robots.
We consider the following setup (\textit{Layout 4}): a 3-rows layout (similar to Fig.~\ref{fig:layout} right) with 10 shelves instead of 5, 5 workstations instead of 2 and 10 waiting places. We solve 100 instances of 10 orders each for 10 robots. Having twice the number of shelves of Layout 3, the graph is about twice as big, the number of robots is more than doubled and the number of tickets is doubled as well in each instance compared to previous experiments.

The results are summarized in Table~\ref{table:10-robots}. Compared with the results for 4 robots, we observe that the median regret is only slightly increased (from 9.3 to 11.2 percent) and the median running time is less than doubled, which shows a reasonable sub-linear scalability of our approach for larger instances.

\begin{table}
  \setlength{\tabcolsep}{4pt}
  \centering
  \begin{tabular}{r|r|r|r}    
    & Makespan & Regret & Run time \\
    \hline
Q1 & 38 min & 9.9 \%  & 1266 s\\ 
Q2 & 45 min & 11.2 \% & 2377 s\\ 
Q3 & 50 min & 12.7 \% & 4939 s\\ 
  \end{tabular}
  \caption{Quartiles of makespan (in minutes), regret (in percent) and run time (in seconds) of the IPP algorithm for 10 robots in the Layout 4.}
  \label{table:10-robots}
\end{table}

\subsection{In Real Environment}\label{sec:real}
To validate the proposed algorithm in realistic use-cases, we performed two experiments in a real warehouse environment whose layout is similar to the left-most layout in Fig.~\ref{fig:layout}.
Its routing graph has 668 nodes and 2103 edges. Each experiment involved two robots. 

The first experiment was conducted with maximum linear speed and maximum angular speed of 0.2m/s and 0.2rad/s, and constant (ac/de)celerations of 0.25[m/s$^2$; rad/s$^2$] for loaded and unloaded robots. 
It involved 4 pickups and 4 drop-offs over 10 minutes. From the experiment, a notable discrepancy was observed between the plan and the real trajectories. Discrepancies were mainly influenced by the stochasticity in the duration that the robots spent picking up or dropping off an object. 

For the second experiment, the estimated pickup and drop-off times were increased by a constant. The robot motion limits were increased to 0.8m/s for linear speed and 0.8rad/s for angular speed, with constant (ac/de)celerations of 0.45[m/s$^2$; rad/s$^2$].
The total set of tasks involved 10 pickups and 10 drop-offs, which were performed continuously over 13 minutes. In the second experiment, the real robot trajectories showed significantly reduced discrepancy from the plans compared to the first experiment. 
The supplementary material contains videos of the two experiments, where we display the the real trajectory data (dark color) as well as the planned trajectory (lighter color) of each robot.

\section{Conclusion and Future Work}
In this paper, we study a lifelong MAPF problem with interdependent tasks, to model a realistic online multi-robot pickup and delivery service in a warehouse. 
We propose the Interleaved Prioritized Planning with VP$^*$ algorithm that takes into account precise robot dynamics to compute collision-free trajectories. 
We show the completeness of IPP-VP$^*$ under simple assumptions and empirically evaluate its feasibility.
In particular, we compare it to an adaptation of existing methods that do not account for the dynamics and show that they fail to provide feasible plans in our context. We conduct a precise ablation study to demonstrate the necessity of the different components of our approach. Finally, we successfully apply IPP-VP$^*$ to a simple scenario in a real warehouse. We believe that our approach makes an important step toward bridging the gap between the extensive literature on MAPF with simplifying assumptions and more complex real-life scenarios. 

The real warehouse experiments reveal that despite the high fidelity of our robots and the controlled environment of an automated warehouse, deviation from the plan is unavoidable, and the introduced time margins may not be sufficient to deal with unexpected delays. 
Hence, a future work includes the ability to replan regularly based on the latest information, which would require speeding-up our planning algorithm. 
Finally, it would be interesting to consider the joint optimization of the task assignment and trajectory planning, extending ideas from previous literature~\citep{okumura_solving_2023, ali_improved_2024} to our more challenging setting.

\clearpage
{
     \bibliographystyle{ieeenat_fullname}
     \bibliography{references}
 }

\clearpage
\appendix
\section*{Appendix}
\subsection*{Implementing move time as a secondary criterion}

To implement time discretization, we need to define the equivalence between two solutions in terms of arrival time.
For this, we propose to split the time horizon in consecutive intervals of identical length and consider two values lying in the same interval as equivalent. 
In our implementation each interval is two seconds long.

Given this definition, we modify the VP* algorithm as follows.
First, we add move time to the weights of the different graph arcs of the routing graph.
This allows for fast computation of the total move time of partial solutions by updating it incrementally as is done for the arrival time.
The total move time is stored in queue elements and is used to compare a tentative new element to all the elements previously added to the queue.
Instead of checking if it has already been added in the queue with the exact same values, we check, when all other values are equivalent, if an element with lower or equal move time has already been added to the queue.
When at destination, we update the minimum arrival time and the minimum move time if the current solution is better than the current best path in terms of arrival time or if it is equivalent and better in terms of move time.
For pruning, when checking if current heap element can improve on current best path, we ensure that the arrival time is lower or, if the arrival times are equivalent, that total move time is lower.

\subsection*{Experiments with move time as a secondary criterion}

To assess the impact of using discretized arrival time combined to move time as a secondary criterion, we redo some of the experiments of Section~\ref{sec:experiments} in this setting.

Table~\ref{tab:states_2crit} compares the number of states, running time and makespan between the previous mono-criterion implementation and the one with move time as a secondary criterion on Layout~0.

\begin{table}
    \setlength{\tabcolsep}{4pt}
    \centering
  \begin{tabular}{r|r||r|r||r|r||r|r}
    \multicolumn{2}{c||}{{}}& \multicolumn{2}{|c||}{{2 robots}} & \multicolumn{2}{|c||}{{3 robots}} & \multicolumn{2}{|c}{{4 robots}}\\
    \hline
    \multicolumn{2}{c||}{{Criteria}}& 2 & 1 & 2 & 1 & 2 & 1\\
    \hline
    \multirow{3}{0.9cm}{Visited states ($\times 10^3$)} 
& Q1 & 3 & 5 & 7 & 21 & 12 & 46\\
& Q2 & 5 & 10 & 12 & 75 & 18 & 74 \\
& Q3 & 7 & 22 & 16 & 122 & 28 & 130 \\ \hline
    
    \multirow{3}{0.9cm}{{Run time (s)}} 
    & Q1 & 4 & 9 & 11 & 38 & 27 & 90  \\ 
& Q2 &  7 & 17 & 18 & 75 & 47 & 137 \\ 
& Q3 &  10 & 39 & 24 & 122 & 79 & 289 \\ \hline
    \multirow{3}{0.9cm}{Make-span (min)} 
    & Q1 & 19 & 18 & 15 & 15 & 13 & 13   \\ 
& Q2 & 22 & 22 & 18 & 17 & 15 & 15 \\ 
& Q3 &  26 & 25 & 19 & 19 & 17 & 15 \\ 
  \end{tabular}
    \caption{Quartiles of number of states visited, run-time and makespan when using moving time as a secondary criterion vs. path time minimization only}
    \label{tab:states_2crit}
\end{table}

Due to the discretization, the number of states visited is reduced in the two-criteria implementation, which reduces the running time of the algorithm.
As for the makespan, we can see that it is hardly impacted.
Of course, having different paths implies different solutions at the level of the prioritized planning algorithm which computes paths repeatedly.
However, in our experiments, those solutions remain close in terms of makespan.

Table~\ref{table:layout_2_crit} shows the same results as for Table~\ref{table:layout} for the larger graphs of Figure~\ref{fig:layout}. 
We can see that, as for Layout~0, we have lower running times and similar regret and makespan metrics, which confirms similar scalability for the two-criteria implementation.

\begin{table}
  \setlength{\tabcolsep}{3pt}
  \centering
\resizebox{0.485\textwidth}{!}{%
  \begin{tabular}{r|r||r|r|r||r|r|r||r|r|r}
 
    \multicolumn{2}{c||}{}& \multicolumn{3}{|c||}{{2 robots}} & \multicolumn{3}{|c||}{{3 robots}} & \multicolumn{3}{|c}{{4 robots}}\\
    \hline
    \multicolumn{2}{c||}{Layout}& {1} & {2} & {3} & {1} & {2} & {3} & {1} & {2} & {3}\\
    \hline
    \multirow{3}{0.9cm}{Make-span (min)}
& Q1 & 23 & 23 & 24 & 18 & 18 & 19 & 16 & 15 & 16\\ 
& Q2 & 29 & 29 & 29 & 22 & 22 & 22 & 19 & 19 & 19\\ 
& Q3 & 32 & 32 & 34 & 25 & 24 & 24 & 23 & 22 & 23\\ \hline 
    \multirow{3}{0.9cm}{{Regret (\%)}}
  &  Q1 & 3.9 & 3.0 & 2.9 & 7.0 & 5.6 & 5.5 & 10.7 & 8.3 & 7.6\\ 
& Q2 & 5.9 & 4.5 & 3.8 & 10.0 & 7.7 & 7.1 & 13.3 & 10.8 & 9.5\\ 
& Q3 & 8.2 & 6.1 & 5.1 & 12.5 & 9.6 & 8.6 & 15.1 & 12.8 & 11.2\\ \hline 
    \multirow{3}{0.9cm}{{Run time (s)}}
    & Q1 & 23 & 18 & 16 & 68 & 51 & 56 & 127 & 131 & 149\\ 
& Q2 & 51 & 37 & 38 & 120 & 107 & 112 & 216 & 223 & 242\\ 
& Q3 & 76 & 67 & 68 & 207 & 165 & 193 & 322 & 343 & 379\\ 
  \end{tabular}}
  \caption{Quartiles of the makespan (in minutes), regret (in percent) and run time (in seconds) of the IPP algorithm for Layouts 1, 2 and~3.}\label{table:layout_2_crit}
\end{table}

\end{document}